\documentclass[hidelinks,letter, 10 pt, conference,doublecolomn]{ieeeconf} 
\IEEEoverridecommandlockouts    
\usepackage[dvipsnames]{xcolor}
\usepackage[T1]{fontenc}

\usepackage{tikz}
\usetikzlibrary{shapes, patterns}
\usetikzlibrary{arrows.meta}
\usepackage{pgfplots}

\usepackage{tikz}
\usetikzlibrary{matrix,positioning,fit,backgrounds}
\usepackage{lettrine}
\usepackage{balance}
\usepackage{graphicx}
\usepackage{longtable}

\usepackage{bm}
\usepackage{amsmath}
\usepackage{amssymb}

\newcommand{\vect}[1]{}
\newcommand{\mat}[1]{}

\newcommand{\diffs}[3]{\frac{\partial^2 #1}{
\ifx#2#3 
\partial #2^2
\else
\partial #2 \partial #3
\fi
}}

\newcommand{\chiv}{\mathbf{\chi}}

\newcommand{\bv}{{b}}

\newcommand{\fv}{{f}}
\newcommand{\gv}{{g}}

\newcommand{\hv}{{h}}
\newcommand{\kv}{{k}}

\newcommand{\pv}{{p}}

\newcommand{\xiv}{{\xi}}

\newcommand{\qv}{{{q}}}
\newcommand{\dqv}{\dot{{q}}}

\newcommand{\rv}{{{r}}}

\newcommand{\uv}{{u}}
\newcommand{\vv}{{v}}

\newcommand{\wv}{{w}}

\newcommand{\xv}{{x}}

\newcommand{\yv}{{y}}

\newcommand{\sigmav}{{\sigma}}

\newcommand{\tauv}{{\tau}}

\newcommand{\Gammam}{{\Gamma}}

\newcommand{\Cm}{\mat{C}}

 \usetikzlibrary{calc}
\usepackage{etoolbox}

\usepackage{svg}
\usepackage{float}
\usepackage{hyperref}
\usepackage{academicons}
\usepackage[dvipsnames]{xcolor}
\usepackage{xcolor}
\usepackage{algorithm}
\usepackage{algpseudocode}

\definecolor{invictusRed}{RGB}{185, 28, 28}
\definecolor{techGrey}{RGB}{156, 163, 175}
\definecolor{darkGraphite}{RGB}{75, 85, 99}
\definecolor{lightConcrete}{RGB}{229, 231, 235}
\definecolor{absBlack}{RGB}{10, 10, 10}

\usetikzlibrary{matrix, shapes.geometric, positioning, backgrounds}

\newcommand{\mirko}[1]{{\color{WildStrawberry}\footnotesize MM: #1}}

\usepackage{siunitx}

\usepackage{amsthm}
\theoremstyle{plain}

\newtheorem{thm}{Theorem}
\newtheorem{lem}{Lemma}

\newtheorem{defn}{Definition}
\newtheorem{prop}{Proposition}

\newtheorem{assumpt}{Assumption}
\newtheorem{problem}{Problem}

\usepackage{comment}
\usepackage[font=small]{caption}
\usepackage{subcaption}
\usepackage{calc}
\newtheorem{rem}{\textbf{Remark}}[section]
\usepackage{mathtools} 
\usepackage{listings}
\usepackage{xcolor}

\renewcommand{\theenumi}{\roman{enumi}}  

\usepackage{booktabs,tabularx,makecell}

\usepackage{environ}

\newif\ifarxiv
\arxivtrue

\author{Mirko Mizzoni$^{1,\orcidlink{0009-0006-2165-3475}}$, Pieter van Goor$^3$\orcidlink{0000-0003-4391-7014}, Barbara Bazzana$^1$\orcidlink{0000-0002-2843-4324}, and  Antonio Franchi$^{1,2,\orcidlink{0000-0002-5670-1282}}$
\thanks{$^1$Robotics and Mechatronics group, Faculty of Electrical Engineering,  Mathematics, and Computer Science (EEMCS), University of Twente, 7500 AE Enschede, The Netherlands. {\footnotesize \tt m.mizzoni@utwente.nl}, {\footnotesize } {\footnotesize \tt schol@r-franchi.eu}}
\thanks{$^2$Department of Computer, Control and Management Engineering, Sapienza University of Rome, 00185 Rome, Italy, {\footnotesize \tt s.orelli@uniroma1.it, schol@r-franchi.eu } {\footnotesize \tt }}
\thanks{$^{3}$ School of Aerospace, Mechanical, and Mechatronic Engineering (AMME), Faculty of Engineering, University of Sydney, NSW, 2006, Australia. {\footnotesize \tt pieter.vangoor@sydney.edu.au}}
\thanks{This work was partially funded by the Horizon Europe research agreement no. 101120732 (AUTOASSESS).}}

\usepackage{orcidlink}
\usetikzlibrary{shapes.geometric, positioning, matrix}
\title{\LARGE  \bf Switching Among Feedback-Linearizing Output Sets (Melds):\\
Dwell-Time and Compatibility Guarantees}

\begin{document}
\maketitle



\begin{abstract}
We study switching among multiple square selections of output functions (melds) drawn from a deck of candidate outputs for nonlinear systems that are static feedback linearizable via outputs. Fixing an operating point, each meld induces a distinct feedback-linearizing coordinate chart defined on a common neighborhood. Switching between melds therefore produces state-dependent coordinate mismatches that are not captured by classical switched-system analyses. We quantify this effect through Lipschitz bounds on the cross-chart maps over a compact safe set and introduce a reference-compatibility constant that measures mismatch among reference families across melds. We derive  an explicit dwell-time condition depending on controller decay rates and the compatibility constant, that guarantees exponential decay of the active-output tracking errors between switches, seamless tracking of outputs shared by consecutive melds, and uniform boundedness of the state error within the safe set. A planar 3R manipulator illustrates the results.
\end{abstract}

\section{Introduction}

Many nonlinear tracking controllers rely on feedback linearization techniques~\cite{Isidori1995, NijmeijerSchaft}: one selects a square output vector and makes its dynamics linear by static state feedback. This approach is widely used in robotics and mechatronics because it yields fast, interpretable tracking behavior and explicit tuning through linear gains~\cite{2016-Correll, 2021g-OllTogSuaLeeFra}.

In many practical systems, however, the number of candidate outputs exceeds the number of available inputs. A robot manipulator, for instance, may have more task-relevant quantities than actuators: multiple end-effector coordinates, joint variables, energy-related outputs, contact-related outputs, or sensing-driven objectives. When objectives change over time—because a supervisor issues new commands, because the environment evolves, or because resources are temporarily constrained—the controller cannot track all outputs simultaneously. Instead, it must choose which subset of outputs to track at each time, and potentially switch among subsets as priorities evolve.

A common way to handle this mismatch is output prioritization through optimization-based or null-space control schemes. These methods are flexible, but they can introduce solver dependence, conservatism in stability guarantees, and performance degradations at transitions~\cite{task-priorit,ref_mpc,He2015OnSO,8814618}.

This motivates a complementary perspective: rather than blending many outputs at once, switch among a finite set of feedback-linearizing output choices, using standard linearizing controllers within each choice.

At first glance, this looks like a standard switched-systems problem: each mode corresponds to a stabilizing closed-loop controller, so perhaps stability follows from existing dwell-time results ~\cite{4395172,5717596,Liberzon2003,TomlinSastry1998}. The difficulty is that in feedback linearization, each output choice induces a different feedback-linearizing coordinate transformation. Switching modes therefore does not merely change the closed-loop dynamics in fixed coordinates; it also changes the coordinates used to represent tracking errors and state deviations. Concretely, the same physical state can be mapped to different linearized coordinates depending on the active output set, and at a switch the "error state” is effectively re-parameterized by a nonlinear, state-dependent change of chart. Classical dwell-time analyses for switched systems typically assume a common state representation (or at least compatible Lyapunov coordinates), and therefore do not directly capture this chart-induced mismatch~\cite{Liberzon2003}.

This work develops constructive stability guarantees for switching among feedback-linearizing output selections.  We consider a nonlinear system that is static feedback linearizable (via outputs) around a fixed operating point. We start from a deck of candidate scalar outputs and focus on square selections of size equal to the number of inputs. We call a square selection a ``meld'' - borrowed from card games -  if it is valid for exact state-space linearization (i.e., it has full vector relative degree and a nonsingular decoupling matrix at the operating point). Each meld defines a local feedback-linearizing coordinate chart on a common neighborhood of the operating point. A switching signal selects one meld at a time, resulting in piecewise-smooth dynamics, where each switching instant induces a nonlinear change of feedback-linearizing coordinates
(see Figure~\ref{fig:scheme}).
The key question is:
\begin{quote}
\emph{
``Under what conditions does switching among such melds preserve boundedness of the overall state and guarantee consistent tracking, despite chart changes at each switch?''}
\end{quote}

We answer this question locally on a compact safe set where all considered melds remain valid; the central difficulty is that switching induces a nonlinear change of feedback-linearizing coordinates, generating state-dependent mismatches at each switching instant.

We quantify this effect through uniform Lipschitz bounds on the inverse chart maps and on the associated cross-chart mappings. In addition, we introduce a reference-compatibility constant that measures the mismatch between reference families across melds. This constant makes explicit that dwell time alone cannot compensate for incompatible references.

Under these ingredients, we derive an explicit minimum dwell-time condition depending on the closed-loop decay rate, the cross-chart Lipschitz bounds, and the reference-compatibility constant. 

The exponential decay of the active-output tracking errors on each dwell interval follows directly from the single-meld design. The main additional guarantees under switching are:
\begin{enumerate}
   \item \emph{Seamless tracking of shared outputs:} any output shared by consecutive melds retains its exponential convergence across the switch;
  \item \emph{Uniform boundedness and safety:} the state remains uniformly bounded with respect to the reference-induced state, and remains in the safe set under a margin condition.

\end{enumerate}

\begin{figure}[t!]
    \centering
    \includegraphics[trim= 0.2cm 0 0 0,clip,width=1.05\linewidth]{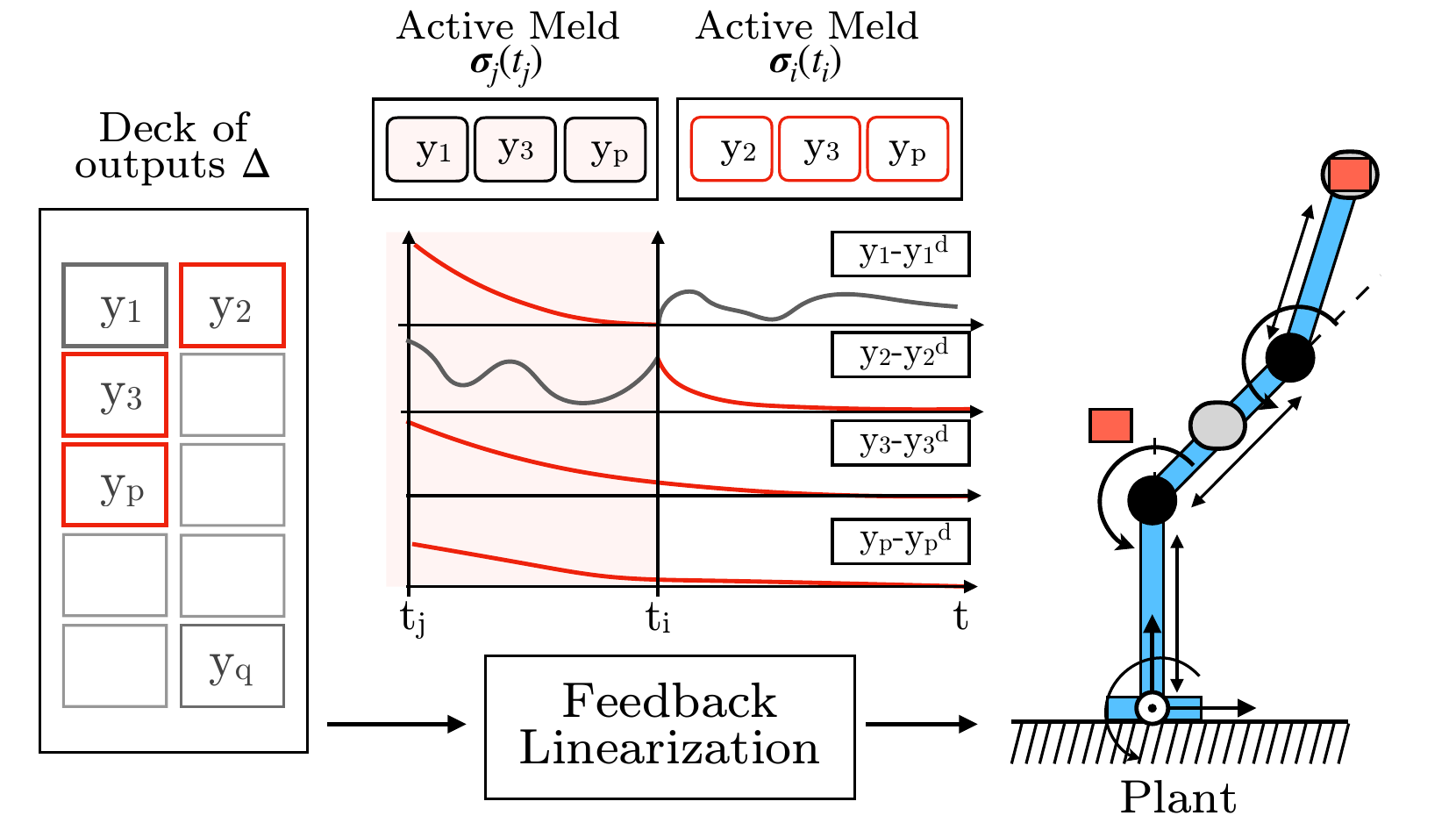}
    \caption{Conceptual illustration of switching among feedback-linearizing output selections (melds). From a deck of candidate outputs $\Delta$, different square selections (melds) are activated at switching instants. Each meld induces a distinct feedback-linearizing coordinate chart. At switching, tracking errors are re-parameterized, yielding transient mismatches for newly activated outputs, while outputs shared across consecutive melds preserve their exponential convergence.}
    \label{fig:scheme}
\end{figure}

This framework generalizes previous work in~\cite{2025a-MizGooFra}, which was limited to only two of the subsets that we call melds here, and did not analyze the stability of the system dynamics under switching. Notice also that the resulting closed-loop system is modeled as a switched system~\cite{Liberzon2003} with externally prescribed switching signals, distinguishing it from hybrid control schemes that rely on state-dependent transitions~\cite{opt_control_h_switched_sys,9662411}.

The outline of the letter is as follows. Section~\ref{subs:preliminaries} introduces the foundational concepts and notation. Sections~\ref{subs:sect_method} and~\ref{sect:feed_control} present the notion of melds and address the \emph{seamless tracking of shared outputs} guarantee. Section~\ref{sect:state_bound} addresses the \emph{uniform boundedness and safety} of the state. Finally, the letter concludes with an application to a 3R robotic platform and supporting simulation results in Section~\ref{sec:rigid-body}.

\section{Preliminaries}\label{subs:preliminaries}

For a comprehensive introduction to feedback linearization, the interested reader is referred to \cite{Isidori1995}.

Consider a multivariable nonlinear system
\begin{equation}
    \Sigma: \quad \dot{\xv}  = \fv(\xv)+\sum_{i=1}^p\gv_i(\xv)u_i
    \label{eq:sys}
    \end{equation}
where \mbox{$\xv\in \mathbb{R}^n$} is the state, $u_1,\ldots,u_p$ are the control inputs,  
$\fv(\xv)$,  $\gv_1(\xv)$, $\ldots$, $\gv_{p}(\xv)$ are smooth vector fields.
The output vector is defined as 
   $ \yv=\hv(\xv)$
with  $\hv(\xv):=\begin{bmatrix} h_1(\xv) \;\; \ldots \;\;h_{p}(\xv)\end{bmatrix}^\top$  and  where \mbox{\(\hv: \mathcal{U} \to \mathbb{R}^p\)} is a smooth mapping on an open set \mbox{$\mathcal{U} \subset \mathbb{R}^n$}.
The  system (\ref{eq:sys}) is said to have \emph{(vector) relative degree} $\rv = [
    r_1 \cdots r_p
]^\top \in \mathbb{N}^p$ at a point $\xv^\circ$ (sometimes denoted with $\rv(\xv^\circ)$) w.r.t. the input-output pair  $(\uv,\yv)$ if  \textrm{i)}
$  L_{\gv_j}L^{k}_{\fv} h_i(\xv) =0,$
for all $1\leq j \leq p$, for all $k< r_i-1$, for all $1\leq i \leq p$ and for all $\xv$ in a neighborhood of $\xv^\circ$, and 
\textrm{(ii)} the $p\times p$ \emph{decoupling} matrix  $A(\xv) := \ \big[ \; L_{\gv_j} L_{\fv}^{\,r_i-1} h_i(\xv) \;\big]_{{i=1,\dots,p;\; j=1,\dots,p}}$
is full rank at $\xv = \xv^\circ$. 
Then the output vector at the $\rv$-th derivative(s) may be rewritten as an affine system of the form
\begin{equation}
{\yv}^{(\rv)} :=[
    y_1^{(r_1)}  \cdots \;  y_{p}^{(r_{p})}
]^\top =\bv(\xv)+A(\xv)\uv,
\label{eq:y_r_now}
\end{equation}
with  $
 \bv(\xv):=\left[
L_{\fv}^{(r_1)}{h_1(\xv)}  \cdots  
L_{\fv}^{(r_{p})}{h_{p}(\xv)}
\right]^\top.$

\subsection{Exact state–space linearization (via output function)}
Suppose the system \eqref{eq:sys} has some (vector) relative degree $\rv$ at $\xv^\circ$ and that the matrix $G(\xv^\circ)=[\gv_1(\xv)\cdots\gv_p(\xv)]$ has rank $p$ in a  neighborhood $\mathcal{U}$ of $\xv^\circ$. 
Suppose also that  \mbox{$|\rv|:=\sum^n_{i=1}r_i=n$}, and choose the  control input to be 
$\uv = A^{-1}(\xv)[-\bv(\xv)+\vv],$
where $\vv \in \mathbb{R}^{p}$ can be assigned freely, and $\bv(\xv)$ is defined as above.
Then the output dynamics \eqref{eq:y_r_now} become
$\yv^{(\rv)} = \vv$.
We refer to  $\yv$ as a \emph{feedback-linearizing output vector}: the system state and input can be expressed in terms of $\yv$ and its time derivatives. 
In the sequel, the term `feedback linearization' is used as short-hand for `exact state-space linearization via output functions'.

\section{Melds of a Control System}
\label{subs:sect_method}

Consider a dynamical system $\Sigma$ as in~\eqref{eq:sys}
satisfying the same regularity assumptions as in Section~\ref{subs:preliminaries}.

\subsection{Deck, Selection, Melds}

We define a \emph{deck} of $q\geq p$ candidate outputs to be a list of $q$ scalar functions of the state $\xv$, i.e.,
\begin{equation}
\Delta = \{ y_i = h_i(\xv) \}_{i=1}^q, \qquad {y}_{\Delta} := [h_1(\xv)\ \cdots\ h_q(\xv)]^\top,
\end{equation}
and let ${r} = [r_1\cdots r_q]^\top$ be their relative degrees at an operating point $\xv^\circ$.
A \emph{square selection} is $\sigmav \in \{0,1\}^q$ with exactly $p$ ones; denote the active indices by $\mathcal{I}_{\sigmav}=\{i^\sigma_1,\ldots,i^\sigma_p\}$ and the selection matrix $\Gamma_{\sigmav}\in\{0,1\}^{p\times q}$ formed by picking the rows indexed by $\mathcal{I}_{\sigmav}$ in the $q\times q$ identity matrix. Then ${y}_{\sigmav} := \Gamma_{\sigmav} {y}_{\Delta} $.

For the whole deck $\Delta$, define the $q\times p$ decoupling matrix
\begin{equation}
A_\Delta(\xv) := \big[L_{\gv_j} L_\fv^{\,r_i-1} h_i(\xv)\big]_{i=1,\dots,q;\ j=1,\dots,p}.
\end{equation}
For a square selection $\sigmav$, the $p\times p$ decoupling matrix is
\begin{equation}
A_{\sigmav}(\xv) := \Gamma_{\sigmav} A_{\Delta}(\xv),\qquad {r}_{\sigmav} := \Gamma_{\sigmav} {r}.
\end{equation}

\begin{defn}[Meld]
    Let $x^\circ \in \mathbb{R}^n$ be  a fixed operating point.
    A square selection $\sigma \in \{0,1\}^q$ with exactly $p$ ones is called a \emph{meld} at $x^\circ$ if i)  $\sum_{i\in \mathcal{I}_{\sigmav}} r_i = n$, and ii) $\det A_{\sigmav}(\xv^\circ)\neq 0$. 
    Formally,
    \[
\mathcal{M}_{\xv^\circ} := \left\{ \sigmav \in \{0,1\}^q : \sum_{i\in \mathcal{I}_{\sigmav}} r_i = n,\ \det A_{\sigmav}(\xv^\circ_{})\neq 0\right\}.
\]

    The set of all melds at $x^\circ$ is denoted with $\mathcal{M}_{\xv^\circ}$.
\end{defn}
\begin{rem}[Local common validity]
Fix $x^\circ$ and let $\sigma \in \mathcal{M}_{x^\circ}$.
Since $\det A_\sigma(x)$ is continuous and $\det A_\sigma(x^\circ)\neq 0$,
there exists an open neighborhood $\mathcal{B}_\sigma$ of $x^\circ$ such that
$\det A_\sigma(x)\neq 0$ for all $x\in \mathcal{B}_\sigma$. Moreover, by standard
feedback-linearization regularity results, the vector relative degree
$r_\sigma$ remains constant on a possibly smaller neighborhood of $x^\circ$.
 Intersecting these neighborhoods if
necessary, we may assume that on $\mathcal{B}_\sigma$ both
properties hold. 
Because $\mathcal{M}_{x^\circ}$ is finite, the intersection
$$
\mathcal{B}:= \bigcap_{\sigma\in\mathcal{M}_{x^\circ}} \mathcal{B}_\sigma
$$
is a nonempty open neighborhood of $x^\circ$ on which all melds are valid.
\end{rem}

\begin{defn}[Safe set]
A safe set is any nonempty compact set $\mathcal{K}\subset \mathcal{B}$.
\end{defn}

Throughout the remainder of the letter, the operating point
$x^\circ$ is fixed and all melds, neighborhoods, and safe sets
are defined with respect to $x^\circ$.

\section{Output Tracking of Switching Melds}
\label{sect:feed_control}

In this section, we formalize the framework for switching between different output sets. Regularity assumptions and invariance properties are stated with respect to a safe set
$\mathcal{K}\subset \mathcal{B}$.

 The core difficulty in switching among melds is that each meld induces a different feedback-linearizing coordinate chart, so a switch causes a (state-dependent) change of coordinates in which tracking errors are represented. Our analysis proceeds in four steps: (i) we define the meld-dependent charts and the associated cross-chart maps;  (ii) we introduce deck-level reference jets and tracking/mismatch errors; (iii) we present a unified control law that reduces to standard input--output linearization on each interval of constant meld; and (iv) we characterize what remains invariant under switching (shared outputs) and formulate the dwell-time problem that controls the chart-induced jumps.

\subsection{Coordinate Transformations and Meld Mappings}
For each output $y_i=h_i(\xv)$ of relative degree $r_i$, define the map which lifts the state to the $(r_ 
i - 1)$-th jet of the $i$-th output trajectory\footnote{We use ``jet'' in the standard sense of jet
bundle coordinates; see, e.g., \cite{Saunders1989}. 
}:
\begin{equation}
\Phi_i:\mathbb{R}^n\to\mathbb{R}^{r_i},\qquad \Phi_i(\xv):=\big[y_i,\ \dot y_i,\ \ldots,\ y_i^{(r_i-1)}\big]^\top=:\overline{\yv}_i,
\end{equation}
where $y_i^{(k)}=L_\fv^{\,k}h_i(\xv)$ for all $k=0,\ldots,r_i-1$.
For a given meld $\sigmav \in \mathcal{M}_{x^\circ}$, stack the selected maps to define 
\begin{equation}
\Phi_{\sigmav}:\mathbb{R}^n\to\mathbb{R}^n,\qquad 
\Phi_{\sigmav}(\xv):=\big[\Phi_{i^{\sigma}_1}(\xv)^\top \ \cdots\ \Phi_{i^{\sigma}_p}(\xv)^\top\big]^\top,
\end{equation}
which is a local diffeomorphism on $\mathcal{B}_{\sigmav}$ by standard exact linearization results.
However, local invertibility guarantees only point-wise (neighborhood-wise) invertibility, and does not by itself provide a globally well-defined inverse on a region of interest, nor uniform bounds needed to control coordinate jumps at switching instants. Since our results are local to a compact safe set $\mathcal{K}$, we explicitly assume that each chart $\Phi_{\sigmav}$ admits a diffeomorphic restriction to $\mathcal{K}$, with Lipschitz inverse and Lipschitz cross-maps, as stated next.

\begin{assumpt}(Regularity).\label{assumpt:common_validity_region_lip_const}
For each $\sigmav \in  \mathcal{M}_{\xv^\circ}$, there exists  an open set $\mathcal{U}_{\sigmav}$ with $\mathcal{K}\subset \mathcal{U}_{\sigmav}$ such that the map
$\Phi_{\sigmav}: \mathcal{U}_{\sigmav}\to \Phi_{\sigmav}(\mathcal{U}_{\sigmav})$ is a diffeomorphism   with inverse $\Psi_{\sigmav}$ \footnote{If a smooth map has a Jacobian that is uniformly nonsingular on a compact set, then the inverse map exists on that set and is Lipschitz with a uniform constant. In particular, since the feedback–linearizing coordinate maps $\Phi_{\sigmav}$ are smooth and nonsingular on $\mathcal{K} \subset \mathcal{B}_{\sigmav}$, their inverses $\Psi_{\sigmav}$ are well defined and uniformly Lipschitz on $\Phi_{\sigmav}(\mathcal{K})$. The same argument applies to the maps $\Theta_{i,\sigmav}$, which are smooth compositions involving $\Psi_{\sigmav}$ on compact domains.}. Moreover, $\Psi_{\sigmav}$ is Lipschitz on $\Phi_{\sigmav}(\mathcal{K})$.
For each $i \in \{1,\dots,q\}$, the cross-map
$$
\Theta_{i,\sigmav}:=\Phi_i\circ \Psi_{\sigmav}
$$ is Lipschitz on $\Phi_{\sigmav}(\mathcal{K})$.
\end{assumpt}

Define the uniform bounds
\begin{equation}
L^\Psi := \max_{\sigmav \in  \mathcal{M}_{s}} L^\Psi_{{\sigmav}}, \qquad
L^\Theta := \max_{\sigmav \in  \mathcal{M}_{s}} \max_{i} L^\Theta_{{i,\sigmav}} .
 \label{eq:constants}
\end{equation}
Physically, the cross-map $\Theta_{i,\sigmav}$ answers the question:
\emph{given the active meld coordinates, what would output $y_i$ (and its derivatives) be?}
That is, $\Theta_{i,\sigmav}(\cdot)$ returns the jet of $y_i$ evaluated at the state implied by the active meld chart. 
This is key for switching: it allows us to compare inactive references with the trajectories implied by the currently active meld and to bound the resulting mismatch.

\subsection{Reference Trajectories and Tracking Errors}
Let $\yv^d(t)\in\mathbb{R}^q$ collect the desired scalar trajectories for the deck, i.e., $\yv^d(t)=[y_1^d(t)\ \cdots\ y_q^d(t)]^\top$. Assume each $y_i^d:[t_0,\infty)\to\mathbb{R}$ is $C^{r_i-1}$ so that the corresponding $(r_i-1)$-jet is well-defined. Define the desired jet for output $i$ as
\begin{equation}\label{eq:desired_jet_def}
\overline{\yv}_i^d(t) := \big[y_i^d(t)\ \dot y_i^d(t)\ \ldots\ y_i^{d\,(r_i-1)}(t)\big]^\top \in \mathbb{R}^{r_i}.
\end{equation}
We define the tracking error for output $y_i$ as the difference between the desired and actual jets:  
   \begin{equation}
       \xiv_i(t) := \overline{\yv}^d_i(t)- \overline{\yv}_i(t) ,
   \end{equation}
 When the output $y_i$ is part of the active meld, ${\xiv}_i$ is the standard tracking error. When the output $y_i$ is \emph{inactive}, $\xiv_i$ represents the \emph{compatibility error} (or reference mismatch), i.e., the deviation between the non-active reference $\overline{\yv}^d_i$ and the trajectory
implied by the active meld. Bounding this error is crucial for stability during switching.

\subsection{Control law}
 We associate a gain vector ${k}_i$ with each output $y_i$, chosen such that the error dynamics are Hurwitz.  We then construct a virtual input ${w} \in \mathbb{R}^q$ for the entire deck, where each entry $w_i$ is of the form :

\begin{equation}
\begin{split}
{w}_i &= \yv_i^{d,(r_i)}+\underbrace{k_i^1(y^d_i-y_i)+\ldots+k^{r_i-1}_i(y_i^{d,(r_i-1)}-y_i^{(r_i-1)})}_{={k}_i^\top \xiv_i}.
\end{split}
\label{eq:w}
\end{equation}
When a particular meld $\sigmav$ is active, the actual control input uses only the components selected by $\Gamma_{\sigmav}$, yielding the standard feedback-linearizing law for that meld.

Given references $\overline{{y}}^d_{\sigmav}:=\big[{\overline{\yv}^d_{i^{\sigma}_1}}^\top \ \cdots\ {\overline{\yv}^d_{i^{\sigma}_p}}^\top\big]^\top,$ for the active meld $\sigmav$, the map \mbox{$\chi: [t_0,\infty)\to \mathbb{R}^n$}, defined as  
\begin{equation}
\chi:=\Psi_{\sigmav}\big(\overline{{y}}^d_{\sigmav}\big)
\end{equation}
corresponds to the desired reference state induced by the meld~$\sigmav$.

     %

\subsection{Recall of Standard Single Output vector Tracking}\label{sect:recall_standard_output_vector}
We briefly recall the single-meld tracking result to extract uniform exponential decay constants $(C,\alpha)$ that will later enter the cross-chart error bounds under switching.
Let us focus initially on a single meld  $\sigmav\in\mathcal{M}_{\xv^\circ}$. We denote with $\yv_{\sigmav} = \Gammam_{\sigmav}\yv$, i.e., the output vector corresponding to the meld $\sigmav$. Then, at the (vector) relative degree {${\rv}_{\sigmav}$}, the output dynamics takes the form 
\begin{equation}
\begin{split}
\yv^{({\rv}_{\sigmav})}_{\sigmav}&=\Gammam_{\sigmav}\left (\left[\begin{smallmatrix}
        b_1(\xv)\\
        \vdots \\
        b_q(\xv)
    \end{smallmatrix}\right]+ A_{\Delta}(\xv)\uv\right)  = \Gammam_{\sigmav} \bv(\xv)+ A_{\sigmav}(\xv)\uv.
\end{split}
\label{eq:output_res}
\end{equation}
By construction,  the matrix $A_{\sigmav}(\xv)$ is a  nonsingular $p\times p$ square matrix. This allows us to design a   control law   of the form
 \begin{equation}
    \uv_{\sigmav} = A_{\sigmav}(\xv)^{-1} (-\Gammam_{\sigmav}\bv(\xv)+\Gammam_{\sigmav} \wv).
  \label{eq:control}%
\end{equation}%
with the entries of $\wv$ as in~\eqref{eq:w}.
This choice leads to $p$-\emph{linear} and \emph{decoupled} subsystems each with state $\xiv_{i}$ and with dynamics 
\begin{equation}
\label{eq:subsystems_decoupled}
    \dot{\xiv}_i = \left[\begin{smallmatrix}
    \begin{matrix}
        {{0}_{r_i-1\times 1}} & 
       {{I}_{r_i-1\times r_i-1}}\\ 
    \end{matrix}\\
    \;-\kv_i^\top
\end{smallmatrix}\right]\xiv_i =: A_i \xiv_i, \qquad \forall i\in \mathcal{I}_{\sigmav_j}.
\end{equation}
Therefore, there exist two constants $C_{\sigmav},\alpha_{\sigmav}>0$  such that
\begin{equation}
\resizebox{\linewidth}{!}{$
||\overline{\yv}^d_{\sigmav}(t)-\overline{\yv}_{\sigmav}(t)||\leq C_{\sigmav}e^{-\alpha_{\sigmav} (t-t_0)}||\overline{\yv}^d_{\sigmav}(t_0)-\overline{\yv}_{\sigmav}(t_0)||\;\; \forall t\geq t_0,
    \label{eq:link}$}
\end{equation}
where $C_{\sigmav},\alpha_{\sigmav}$ depends on the chosen gains.

Repeating the same reasoning for all the possible melds in the deck we obtain $2|\mathcal{M}_s|$ constants. Let us denote with $\alpha$ the time constant corresponding to the slowest eigenvalue chosen via the control gains for all the possible melds, i.e.,
\begin{equation}
    \alpha = \min_{j =1,\ldots,|\mathcal{M}_{s}|}\alpha_{\sigmav_j}
    \label{eq:alpha_min}
\end{equation} and with  $C$ the maximum constant for all possible melds, i.e.,
\begin{equation}
C = \max_{j =1,,\ldots,|\mathcal{M}_{s}|}C_{\sigmav_j}.
\label{eq:C_max}
\end{equation}

\subsection{Seamless Output Tracking of Switching Melds}

Departing from the state of the art, let us now consider the case of multiple melds. First, we need to introduce the concept of switching signal.

\begin{defn}[Switching signal]\label{def:switching_signal}
Let $\mathcal{T}=\{t_k\}_{k\in\mathbb N_0}$ be a strictly increasing sequence of switching instants with $t_0\ge 0$
and $t_k\to\infty$.
Fix a set of admissible melds $\mathcal{M}_{s}\subseteq~\mathcal{M}_{\xv^\circ}$.
A \emph{switching signal} is a piecewise-constant map
$\sigmav:[t_0,\infty)\to\mathcal{M}_{s}$
defined by $\sigmav(t)=\sigmav_k$ for \mbox{$t\in[t_k,t_{k+1})$}, where $\sigmav_k\in\mathcal{M}_{s}$.
\end{defn}

Unlike state-dependent switching strategies (e.g., sliding mode or guard conditions),
here the switching signal is exogenous.
At this point, the control law~\eqref{eq:control} can be applied by replacing the fixed selection $\sigmav$ with the switching signal $\sigmav(t)$.
Then, the properties mentioned in the previous section hold on each interval $[t_k,t_{k+1})$ where a single meld is active.
In particular, for $i\in \mathcal{I}_{\sigmav(t)}$ the error subsystem associated with $\xiv_i$ is exponentially stable on $[t_k,t_{k+1})$.

\paragraph{Output Shared across multiple melds}
It is interesting to examine  the behavior of the outputs shared across two or more consecutive melds of the signal $\sigmav(t)$ under this control law.  
To this aim, let us introduce the set of all outputs shared across multiple melds. For a given $k,l\in \mathbb{N}_0$,  consider the block of $l$ consecutive intervals 
$[t_k, t_{k+l+1}) = \bigcup_{j=k}^{k+l} [t_j, t_{j+1})$ with $t_j \in \mathcal{T}$.  
Then, define the set of output indices that remain active throughout this block as 
\begin{align}\label{eq:active_outputs_over_indices}
\mathcal{S}_{k,l} := \bigcap_{j=k}^{k+l}\mathcal{I}_{\sigmav(t_j)}.
\end{align}

\begin{prop}
Let $(\Sigma,\Delta)$ and a switching signal $\sigmav(t)$ taking values in $\mathcal{M}_{s}$. Fix $k \geq 0$ and $l \in \mathbb{N}_0$, and consider the set $\mathcal{S}_{k,l}$ defined in \eqref{eq:active_outputs_over_indices}. 
Then, under the control law \eqref{eq:control}, the subsystem associated  with $\xiv_i$ evolves on $[t_k, t_{k+l+1}) $ according to the same dynamics as in each individual meld interval, i.e., it is not affected by the switching. In particular, its exponential stability is preserved over $[t_k, t_{k+l+1}) $.
\end{prop}
\begin{proof}
 Fix $i\in \mathcal{S}_{k,l}$ and
    consider the time interval $[t_k,t_{k+l+1})$.  
    We have seen in Sect.\ref{sect:recall_standard_output_vector} that  for every meld - under the control law \eqref{eq:control}, the  system dynamics can be decomposed in $p$ subsystems, each of the form shown in \eqref{eq:subsystems_decoupled}.
At every switching instant, the subsystem associated with $\xiv_i$ remains unaffected and continues with the same dynamics thereafter. This completes the proof.
\end{proof}

Secondly, it is crucial to understand and characterize the behavior of the whole state under the proposed switching control law. 
The previous results ensure that, in between the switching time instants, when only one meld is active, $\xv(t)$ converges exponentially to $\chiv(t)$. 
However, switching between stable closed-loop systems does not guarantee overall stability. Even when each subsystem is exponentially stable, the switching itself can lead to unbounded state growth~\cite{Liberzon2003}. The stability of a switched system can sometimes be guaranteed by enforcing a lower bound on the interval between consecutive switches, referred to as a \emph{minimum dwell time}. This consideration naturally leads to the following problem.

\section{Guarantees on the Boundedness of the state}\label{sect:state_bound}
The remaining question is whether switching among exponentially stable meld controllers preserves boundedness of the full state. Since switching induces coordinate mismatches, stability cannot be inferred from single-meld convergence alone. We therefore seek explicit dwell-time conditions that:
\begin{enumerate}[1.]
    \item bound the tracking error uniformly across switches
    \item quantify the effect of cross-chart Lipschitz mismatch, and 
    \item guarantee that the closed-loop state remains inside the safe set $\mathcal{K}$.
\end{enumerate}

\begin{assumpt}\label{assumpt_output_bounded}
    Denote with $\xiv_{i,\sigmav}^d:=\overline{\yv}^d_i - \Theta_{i,\sigmav}(\overline{\yv}^d_{\sigmav})$.
    We assume there exists a constant $N>0$ such that 
    $$
    ||\xiv_{i,\sigmav}^d(t)||_{\infty} \leq N,
    $$
    for all $i=1,..,q$ and $\sigma \in \mathcal{M}_{s}$.
\end{assumpt}

When meld $\sigmav$ is active at time $t$, the active references $\overline{{y}}^d_{\sigmav}(t)$ induce a reference state $\xv_{\sigmav}^d(t):=\chi(t)$. For any (possibly inactive) output $y_i$, the reference that is \emph{implied} by the active meld is $\Theta_{i,\sigmav}(\overline{{y}}^d_{\sigmav(t)})=\Phi_i(\xv_{\sigmav}^d(t))$. The vector $
\xiv^d_{i,\sigmav} := \overline{{y}}^d_i - \Theta_{i,\sigmav}\big(\overline{{y}}^d_{\sigmav}\big)
$
is the \emph{reference mismatch across melds}. Assumption~\ref{assumpt_output_bounded} bounds this mismatch uniformly by $N$, which quantifies the \emph{compatibility of the reference families} across the scheduled melds.

We call $N$  the \emph{reference-compatibility constant.} 
\begin{rem}\label{rem:bound_N}
    The bound 
$N$ is minimized when inactive references agree with the references implied by the active meld, i.e., $
\overline{\yv}_{d,i}(t)\approx \Theta_{i,\sigmav(t)}(\overline{\yv}_{d,\sigmav(t)}(t)).
$
This can be enforced by defining the reference trajectories of inactive outputs so that they remain compatible, over time, with the evolution implied by the active meld, rather than prescribing them independently.
\end{rem}

\begin{lem}\label{lem_bound_y}
Suppose Assumptions~\ref{assumpt:common_validity_region_lip_const}, and~\ref{assumpt_output_bounded} hold.
    Consider any $j=1,\ldots ,|\mathcal{M}_{s}|$
    and any
    $t\geq t_k $ such that $\sigmav(t') = \sigmav_j$ $\forall t'\in[t_k,t)$.  Then, for every $i=1,\ldots ,q$ we have
\begin{equation}\label{eq:bound_output_not_active}
          || \xiv_i(t)||\leq L^\Theta Ce^{-\alpha (t-t_k)} ||\overline{\yv}^d_{\sigmav_j}(t_k)-\overline{\yv}_{\sigmav_j}(t_k)||+N, 
    \end{equation}
    where $\alpha$, $C$, and $L^\Theta$ are as defined in 
    \eqref{eq:alpha_min}, \eqref{eq:C_max}, and \eqref{eq:constants}, respectively.
\end{lem}

\begin{proof}
Consider a generic output $y_i$  and let $\yv_{\sigmav_j}$ be the selected output vector.  In particular,  let us consider $\xiv_i$  which can be rewritten at a given time $t$ as: 
\begin{align*}
\xiv_i(t)
&= \overline {\yv}_i^d(t)-\Theta_{i,{\sigmav}_j}\big(\overline{\yv}_{\sigmav_j}(t)\big) \\
&= \underbrace{\overline{\yv}_i^d(t)-\Theta_{i,\sigmav_j}\big(\overline{\yv}_{\sigmav_j}^d(t)\big)}_{=:~\xiv^d_{i,\sigmav_j}(t)}
 +\underbrace{\Theta_{i,\sigmav_j}\big(\overline{\yv}_{\sigmav_j}^d(t)\big)-\Theta_{i,\sigmav_j}\big(\overline{\yv}_{\sigmav_j}(t)\big)}_{\text{Lipschitz term}}.
\end{align*}

Taking the norms: 
\begin{align}
\|\xiv_i(t)\|
&\le \|\Theta_{i,\sigmav_j}(\overline \yv_{\sigmav_j}^d(t))-\Theta_{i,\sigmav_j}(\overline \yv_{\sigmav_j}(t))\|
   +\| \xiv^d_{i,\sigmav_j}(t)\| \nonumber\\
&\le L^\Theta_{i,\sigmav_j}\,\|\overline \yv_{\sigmav_j}^d(t)-\overline \yv_{\sigmav_j}(t)\| + N.
\end{align}

Then using \eqref{eq:link} and considering $\alpha$ and $C$ instead of the two meld-dependent constants, it yields:
\begin{equation}
       ||\xiv_i(t)|| \leq L^\Theta Ce^{-\alpha (t-t_k)} ||\overline{\yv}^d_{\sigmav_j}(t_k)-\overline{\yv}_{\sigmav_j}(t_k)||+N.
   \notag   
\end{equation}
Due to the definition of $L^\Theta$ in~\eqref{eq:constants} we obtain~\eqref{eq:bound_output_not_active}.
This concludes the proof.
\end{proof}

Lemma~\ref{lem_bound_y} applies not only to the selected outputs of a given meld but to all outputs in the deck.

Moreover, it provides an exponential bound on the tracking error during a fixed meld interval, but the bound still depends on the error at the last switching instant. To ensure that switching does not lead to accumulation of error across intervals, we must impose conditions on the dwell times. In particular, we seek dwell-time bounds that guarantee uniform error bounds independent of the switching instant~$t_k$. This is formalized in the following theorem.

\begin{thm}\label{thm:output_bound}
   Let be given the pair $(\Sigma,\Delta)$  and a switching signal $\sigmav(t)$ taking values in $\mathcal{M}_{s}$. Suppose
Assumptions~\ref{assumpt:common_validity_region_lip_const}, and  \ref{assumpt_output_bounded} hold.
For all $t\in [t_k,t_{k+1})$ while $\sigmav(t)$ is active, the tracking error for each $i\in\{1,\dots,q\}$ satisfies: 
\begin{equation}\label{eq:exp_bound}
\|  \xiv_i(t) \|
\le L^\Theta C e^{-\alpha (t-t_k)}  \sum_{a\in \mathcal{I}_{\sigmav(t)}}\|\xiv_a(t_k)\|
 + N,
\end{equation}
where $L^\Theta,C,\alpha,N$ are the constants appearing in Lemma~\ref{lem_bound_y}.

Moreover, for all  $\epsilon>0$ if the switching instants satisfy:
\begin{equation}
t_1-t_0\geq \tau_0, \;\;\tau_0 := \frac{1}{\alpha}\ln \left (\frac{L^\Theta C \sum_{a\in \mathcal{I}_{\sigmav(t_0)}}\|{\xiv}_a(t_0)\| }{\epsilon}\right),
    \label{eq:delta0}
\end{equation}
and 
\begin{equation}
    t_{k+1}-t_k\geq\tau_k, \quad {\tau_k=\overline\tau}:=\; \frac{1}{\alpha}\ln\!\Big(\frac{L^\Theta C \, p(\epsilon+N)}{\epsilon}\Big), k\geq 1,
    \label{eq:deltatk}
\end{equation}
then, for all $i = 1,\ldots,q$ and for all $k\geq 0$: 
\begin{equation}
    \|\xiv_i(t) \|\le \epsilon + N,\qquad \forall t\in[t_k+\tau_k,t_{k+1}).
    \label{eq:claimed_bound_output}
\end{equation}
\end{thm}

\begin{proof}
    At any $t_k$, it trivially holds: 
    $$
||\overline{\yv}^d_{\sigmav(t)}(t_k)-\overline{\yv}_{\sigmav(t)}(t_k)||\leq  \sum_{a\in \mathcal{I}_{\sigmav(t)}}\|\xiv_a(t_k)\|.
    $$
    By Lemma~\ref{lem_bound_y},  $\forall t\in[t_k,t_{k+1})$,  as long as $\sigmav(t)$ is active : 
   \begin{equation}
       \begin{split}
       ||\xiv_i(t) ||&\leq L^\Theta Ce^{-\alpha (t-t_k)} ||\overline{\yv}^d_{\sigmav(t)}(t_k)-\overline{\yv}_{\sigmav(t)}(t_k)||+N \\
       &\leq L^\Theta Ce^{-\alpha (t-t_k)} \sum_{a\in \mathcal{I}_{\sigmav(t)}}\|\xiv_a(t_k)\|
+N
       \end{split}
   \end{equation}
   for $i=1,...,q$. Consider $k=0$ and let us compute how much time $\tau_0$ is needed  such that the error decreases at most  the quantity $\epsilon+N$ i.e.,
 \begin{equation}
     \begin{split}
         ||\xiv_i(t_0+\tau_0)|| &\leq L^\Theta Ce^{-\alpha \tau_0}\sum_{a\in \mathcal{I}_{\sigmav(t_0)}}\|\xiv_a(t_0)\|+N
\\ &\leq \epsilon +N  
     \end{split}
 \end{equation}
   implying i) $\tau_0$ as in \eqref{eq:delta0} and 
    ii) $\forall t\geq t_0+\tau_0$ as long as $\sigmav(t_0)$ is active: 
 \begin{align}
         ||\xiv_i(t)  || &\leq  L^\Theta Ce^{-\alpha (t-t_0)}\sum_{a\in \mathcal{I}_{\sigmav(t_0)}}\|{\xiv}_a(t_0)\|+N
\\ &\leq L^\Theta Ce^{-\alpha \tau_0}\sum_{a\in \mathcal{I}_{\sigmav(t_0)}}\|{\xiv}_a(t_0)\|+N\\ &\leq
\epsilon +N,
 \end{align}
 due to the monotonic decreasing of the exponential. 
Hence, at any switching instant $t_1\geq t_0+\tau_0$, the initial error for each output , i.e., at $t_1$,  is bounded by $\epsilon+N$ i.e.,  in the subsequent interval,  $\forall t\in[t_1,t_2 )$ we have
\begin{align}
||\xiv_i(t)  ||
&\leq L^\Theta Ce^{-\alpha (t-t_{1})}\sum_{a\in \mathcal{I}_{\sigmav(t_1)}} ||{\xiv}_a(t_1)||+N \label{eq:pre_start}\\
&\leq L^\Theta Ce^{-\alpha (t-t_{1})}\sum_{a\in \mathcal{I}_{\sigmav(t_1)}}  (\epsilon+N)+N\\
&\leq L^\Theta Ce^{-\alpha (t-t_{1})} p(\epsilon+N)+    N.
\label{eq:start_induction}
        \end{align}
    Let us compute how much time $\tau_1$ is needed  such that the error decreases at most  the quantity $\epsilon+N$ i.e.,
\begin{equation}
\begin{split}
||\xiv_i(t_1+\tau_1) || &\leq L^\Theta C e^{-\alpha \tau_1} p(\epsilon+N) +    N\\ &\leq \epsilon + N,
        \end{split}
    \end{equation}
implying i) $\tau_1= \frac{1}{\alpha}\ln\left (\frac{L^\Theta C {p (\epsilon+N)}}{\epsilon}\right)$, and ii) $\forall t\geq t_1+\tau_1$ as long as $\sigmav(t_1)$ is active: 
     \begin{align}
         ||\xiv_i(t)  || &\leq L^\Theta Ce^{-\alpha (t-t_1)}\sum_{a\in \mathcal{I}_{\sigmav(t_1)}}\|{\xiv}_a(t_1)\|+N
\\ &\leq  L^\Theta Ce^{-\alpha \tau_1}\sum_{a\in \mathcal{I}_{\sigmav(t_1)}}\|{\xiv}_a(t_1)\|+N \\ &\leq
\epsilon +N.  
     \end{align}
If at the switching instant $t_{k}$, the initial error is bounded, it can be shown  with analogous passages to \eqref{eq:pre_start}-\eqref{eq:start_induction} that \mbox{$\forall t\in[t_k+\tau_k, t_{k+1})$, it yields
 $
 ||\xiv_i(t) ||  \le \epsilon+N,
$}
with
$
\tau_k = \overline \tau =  \frac{1}{\alpha}\ln \left(\frac{L^\Theta C\, p (\epsilon+N)}{\epsilon}\right)
$, which proves the claimed bound by induction.
 \end{proof}

We are now ready to state the main result, which combines the dwell-time property of Theorem~\ref{thm:output_bound} with the assumption on the state-induced reference trajectory to guarantee uniform boundedness and invariance.

\begin{assumpt}(Reference margin)\label{assumpt:state_induced}
The reference-induced state $\chi(t)$ is well-defined and satisfies
$\chi(t)\in \mathcal{K}$ for all $t\ge t_0$. Moreover, there exists $\delta>0$
such that
\[
\{\xv\in\mathbb R^n:\|\xv-\chi(t)\|\le \delta\}\subset \mathcal{K},
\qquad \forall t\ge t_0 .
\]
\end{assumpt}
This excludes reference trajectories that approach the boundary of $\mathcal{K}$, where charts and decoupling may fail.

\begin{thm}\label{thm:state_error_from_output_with_same_dwell}(Uniform boundedness)
Consider the system $(\Sigma,\Delta)$ under the control law~\eqref{eq:control}, and a switching signal $\sigmav(t)$ taking values in $\mathcal{M}_{s}$.  Suppose 
Assumptions~\ref{assumpt:common_validity_region_lip_const}-~\ref{assumpt:state_induced}  hold  and that the initial condition satisfies
$
\xv(t_0)\in \mathcal{K} .
$
Assume the dwell-time  conditions \eqref{eq:delta0},\eqref{eq:deltatk}  hold for a chosen
$\varepsilon>0$, and define $T := \tau_0$ as in~\eqref{eq:delta0}.

Define the exit time after the transient
$$
t^* := \inf\{\, t \ge t_0 + T : \xv(t) \notin \mathcal{K}\,\}
\in (t_0+T,\infty).
$$
Then the following statements hold.
\begin{enumerate}
    \item For all $t \in [t_0+T,\, t^*)$,
$$
\|\xv(t) - \chiv(t)\| \le S,
$$
where one admissible choice of $S$ is
$$
S := \max\{S_1, S_2\},
$$
with
\begin{equation}
    S_1 := p^2 L^\Psi L^\Theta C (\varepsilon + N) + pN,
\qquad
S_2 := L^\Psi p (\varepsilon + N).
\label{eq:s_t}
\end{equation}

\item If $S \le \delta$, then $t^* = \infty$.
Consequently,
$$
\xv(t) \in \mathcal{K}, \qquad \forall t \ge t_0 + T.
$$
\end{enumerate}
\end{thm}

\begin{proof}
By definition of $t^*$, we have $\xv(t) \in \mathcal{K}$ for all
$t \in [t_0+T,\, t^*)$.
Since Assumption~\ref{assumpt:state_induced} ensures $\chi(t)\in \mathcal{K}$ for all $t \ge t_0$,
both $\xv(t)$ and $\chi(t)$ lie in $\mathcal{K}$ on this interval.
Hence the inverse charts $\Psi_{\sigma(t)}$ and cross-maps
are well defined and Lipschitz with constants $L^\Psi$ and
$L^\Theta$.

For any $t \in [t_0+T,\, t^*)$,
\begin{subequations}\label{eq:basic_reduce}
\begin{align}
\|\xv(t)-\chiv(t)\|
&=\big\|\Psi_{\sigmav(t)}(\overline{\yv}_{\sigmav(t)}(t))-\Psi_{\sigmav(t)}(\overline{\yv}^d_{\sigmav(t)}(t))\big\|
\\
&\le L^\Psi_{\sigmav(t)}\big\|\overline{\yv}_{\sigmav(t)}(t)-\overline{\yv}^d_{\sigmav(t)}(t)\big\|
\\
&\le L^\Psi \big\|\overline{\yv}_{\sigmav(t)}(t)-\overline{\yv}^d_{\sigmav(t)}(t)\big\|
\\
&\le L^\Psi \sum_{i\in \mathcal{I}_{\sigmav(t)}} \|\xiv_i(t)\|.
\end{align}
\end{subequations}

On each switching interval $[t_k,t_{k+1})$ intersected with
$[t_0+T,t^*)$, split the interval into
the early part $[t_k,t_k+\tau_k)$ and late part
$[t_k+\tau_k,t_{k+1})$.

On the late part, Theorem~\ref{thm:output_bound} yields
$\|\xi_i(t)\|\le \varepsilon+N$ for all $i=1,\ldots,q$,
which implies
$$
\|x(t)-\chi(t)\| \le L^\Psi p(\varepsilon+N) = S_2.
$$

On the early part i.e., $t\in[t_k,t_k+\tau_k)$, using Lemma~\ref{lem_bound_y},we have
\begin{equation}\label{eq:early}
\begin{aligned}
\|\xv(t)-\chiv(t)\|
&\le p L^\Psi L^\Theta C e^{-\alpha (t-t_k)}  \\
&\quad \times
\big\|\overline{\yv}^d_{\sigmav(t_k)}(t_k)
      -\overline{\yv}_{\sigmav(t_k)}(t_k)\big\|
+ pN .
\end{aligned}
\end{equation}

Moreover, by Theorem~\ref{thm:output_bound} applied on the previous interval:
\begin{equation}\label{eq:switch_mismatch}
\big\|\overline{\yv}^d_{\sigmav(t_k)}(t_k)-\overline{\yv}_{\sigmav(t_k)}(t_k)\big\|
\le p(\varepsilon+N).
\end{equation}
Combining \eqref{eq:early}--\eqref{eq:switch_mismatch} and using $e^{-\alpha (t-t_k)}\le 1$ yields $\forall t\in[t_k,t_k+\tau_k)\cap[t_0+T,t^*),$
\begin{equation}
\|\xv(t)-\chiv(t)\|
\le p^2 L^\Psi L^\Theta C(\varepsilon+N) + pN
=S_1.
\end{equation}
Therefore,
$$
\|\xv(t)-\chi(t)\| \le S
\quad
\forall t\in [t_0+T,\, t^*).
$$
If $S\le\delta$, then for all
$t\in [t_0+T,\, t^*)$,
$$
\|\xv(t)-\chi(t)\|\le S\le\delta
\quad\Rightarrow\quad
\xv(t)\in \mathcal{K}.
$$
By continuity this also holds at $t=t^*$,
contradicting the definition of $t^*$ unless
$t^*=\infty$.
Hence $\xv(t)\in \mathcal{K}$ for all $t\ge t_0+T$,
and the bound holds globally after $T$.
\end{proof}

   \paragraph{The role of $\tau_0$ and $S$} The constant  $T$ represents the worst-case transient time required for the effect of the initial condition to decay, under the first active meld,   below the level $\epsilon+N$. After time  $T$, the closed-loop system enters a regime in which all subsequent switching transients are uniformly bounded by the dwell-time condition~\eqref{eq:deltatk}, yielding the invariant tube of radius~$S$.  

\paragraph{Achieving $S\leq \delta $}
    The invariance condition 
 in Theorem~\ref{thm:state_error_from_output_with_same_dwell} is not guaranteed solely by increasing the dwell times. From~\eqref{eq:s_t}, even in the limit 
$\epsilon \to 0$ the bound 
$S$ admits a strictly positive lower limit proportional to the constant 
$N$. Therefore, a necessary condition for invariance is that 
$N$ be sufficiently small relative to the safety margin  $\delta$
. Dwell time can attenuate switching transients but cannot compensate for incompatible reference families across melds.

Finally, Theorem~\ref{thm:state_error_from_output_with_same_dwell} shows that,
under the control law~\eqref{eq:control}, any switching signal satisfying the
dwell-time conditions~\eqref{eq:delta0}--\eqref{eq:deltatk} guarantees a uniform
bound on the state error $\|x(t)-\chi(t)\|$ \emph{up to the exit time} from the safe
set $\mathcal{K}$. Moreover, if the resulting bound $S$ satisfies $S\le\delta$, then $\mathcal{K}$ is
forward invariant after the transient $T$, i.e., $x(t)\in \mathcal{K}$ for all $t\ge t_0+T$,
and the same uniform bound holds for all $t\ge t_0+T$.

\begin{figure}[t]
    \centering
\includegraphics[width=1.02\linewidth]{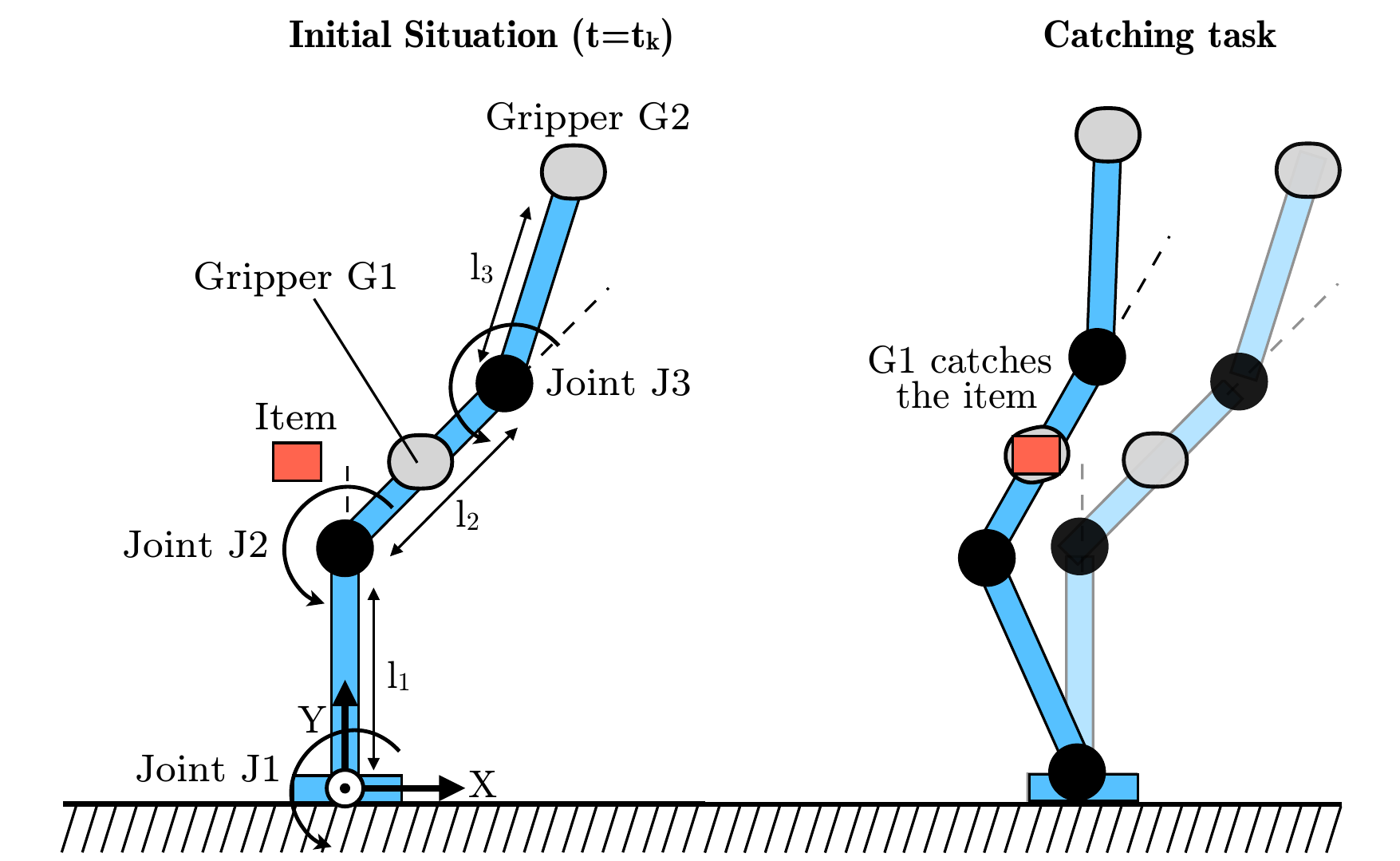}
\caption{ 3R planar manipulator task sequence. (Left) Initial configuration with joint angles and link lengths indicated. (Right) Task execution: the robot reconfigures to align the vacuum gripper $G_1$ with the target item for aspiration. The outputs selected by the active meld are $x_{G_1},y_{G_1}$ and $q_3$.}
\label{fig:3r_robot}
\end{figure}

\section{Numerical Example with a 3R Robot and a Multi-dimensional Deck of Outputs}\label{sec:rigid-body}
We consider a 3R planar manipulator placed on a table and equipped with two vacuum grippers tasked with aspirating items placed by a human  at different locations  and at different  time instants (see Figure~\ref{fig:3r_robot}).
The robot has three revolute joints, hence each joint coordinate is an angle
$q_i \in \mathbb{S}^1$, $i=1,2,3$, and the configuration space is $\mathcal{Q}=(\mathbb{S}^1)^3$.
Since all results in this work are local and hold on a compact safe set away
from singularities, we fix a local chart on $\mathcal{Q}$ and identify it with an open set
of $\mathbb{R}^3$ (i.e., we work with unwrapped angles in a neighborhood of the
operating region). With this local representation we write
$\qv=[q_1\ q_2\ q_3]^\top \in \mathbb{R}^3$ and $\dot \qv \in \mathbb{R}^3$, and define
the state $\xv=[\qv^\top\ \dot \qv^\top]^\top \in \mathbb{R}^6$.
The dynamics are written in state–space form as
\begin{equation}
    \dot{\xv} = \begin{bmatrix} \dot{\qv} \\ -{M}^{-1}(\qv)  {\Cm}(\qv,\dqv)){\dqv} \end{bmatrix}+ \begin{bmatrix}
        {0}\\
        M^{-1}(\qv)
    \end{bmatrix}\tauv
\end{equation}
where $\tauv \in \mathbb{R}^3$
 is the vector of joint torques. The inertia matrix 
$M(\qv)$ is symmetric positive definite, and $\Cm(\qv,\dqv)$ denotes the Coriolis and centrifugal terms. Gravity terms are identically zero due to horizontal planar operation. Explicit expressions for 
$M$ and $\Cm$ follow standard formulations and are omitted for brevity~\cite{FoundationRobotics}.

\begin{figure}[t]
    \centering
\includegraphics[width=1.03\linewidth]{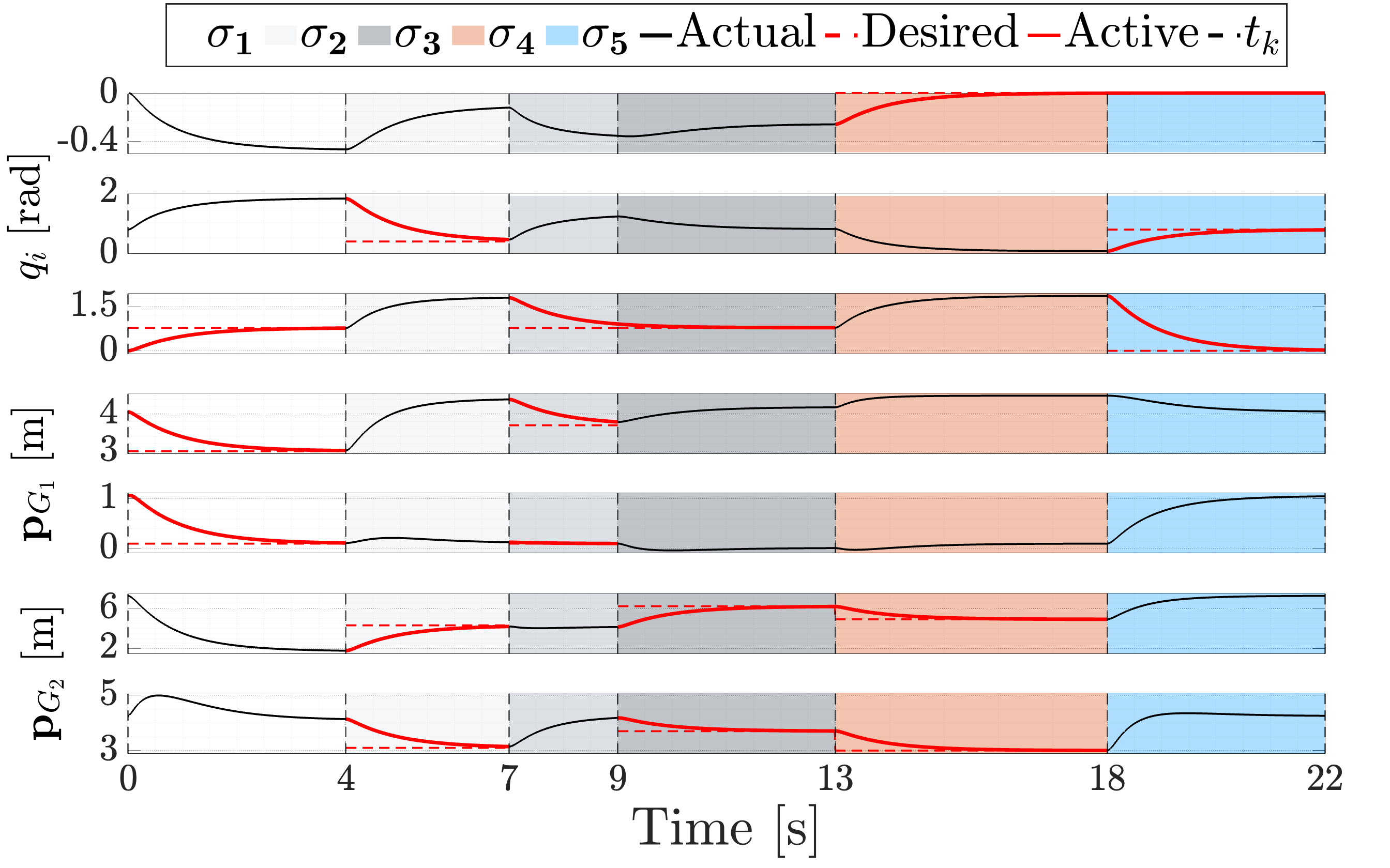}
    \caption{In red: outputs of the meld selected in each interval. The five colored regions correspond to the five melds. The reference trajectories are shown by the \emph{dashed} lines. Outputs shared by consecutive melds keep their exponential tracking across the switch.  }
    \label{fig:Sim}
\end{figure}

The deck of candidate outputs is 
\begin{equation}
\Delta= \left\{
\begin{aligned}
    y_1=q_1, y_2=q_2,y_3=q_3,\;
    y_4= x_{G_1},\;
    y_5= y_{G_1},\;\\
    y_6= x_{G_2},\;
    y_7= y_{G_2}
\end{aligned}\right\},
\end{equation}
where $(x_{G_1},y_{G_1})=:\pv_{G_1}$, and $(x_{G_2},y_{G_2})=:\pv_{G_2}$ denote the Cartesian positions of the two grippers respectively given by: 
\begin{equation}
    \begin{split}
        \pv_{G_1}&=\begin{bmatrix}
            l_1c_1+\tfrac{l_2}{2}c_{12}\\
            \;l_1s_1+\tfrac{l_2}{2}s_{12}
        \end{bmatrix}, 
     \pv_{G_2}=\begin{bmatrix}x_{G_1}+ \tfrac{l_2}{2}c_{12} +l_3c_{123} \\y_{G_1}+ \tfrac{l_2}{2}s_{12} +l_3s_{123}\end{bmatrix}, 
    \end{split}
\end{equation}
where $c_{ijk}=\cos(q_i+q_j+q_k)$, $s_{ijk}=\sin(q_i+q_j+q_k)$, and the virtual input array $\wv$ has entries $w_i$ of the form:
$$
w_i=\ddot{y}_i^{\;d}+k_i^1(\dot{y}_i^d-\dot{y}_i)+k_i^0(y_i^d-y_i), \quad k_i^j>0\;,\:\:\: i=1,...,7.
$$
The gains are identical for all outputs i.e., $k^0_i=k^1_i=15$. 

\emph{Meld schedule (indices).}
We use five compatible melds: 
$\sigmav_1=\{q_3, x_{G_1}, y_{G_1}\}$, 
$\sigmav_2=\{q_2, x_{G_1}, y_{G_1}\},\sigmav_3=\{q_2, x_{G_2}, y_{G_2}\}$, $
\sigmav_4=\{q_1, x_{G_2}, y_{G_2}\}$, and 
$\sigmav_5=\{q_1, q_2, q_3\}$,
with $p=3$. 
All outputs have relative degree two, and the considered square selections form melds with full vector relative degree on a compact set  away from kinematic singularities.


\paragraph{Description of the task }
The items are placed at times $t_1 = 4\si{s}$, $t_2 = 7\si{s}$, $t_3 = 9\si{s}$, $t_4 = 13\si{s}$, and $t_5 = 18\si{s}$.
To successfully grasp each item, the closest gripper is selected and moved above the item. This requires selecting a suitable \emph{meld}, which contains the two coordinates of the selected gripper. The third  output is any joint variable that the overall choice is a meld. 
The manipulator is initialized with joint positions and velocities
$\qv(0) = [0\;\tfrac{\pi}{4}\;0]^\top\si{rad}$ and $\dot{\qv}(0) = [0.1\;\; 0.1\;\: 0.1]^\top\si{rad/s}$. 
Figure~\ref{fig:Sim} illustrates the application of the theory presented so far. The platform operates according to 5 melds, $\sigmav_1=[\begin{smallmatrix}0\;0\;1\;1\;1\; 0\;0\end{smallmatrix}]$, $\sigmav_2=[\begin{smallmatrix}0\;1\;0\;1\;1\; 0\;0\end{smallmatrix}]$, $\sigmav_3=[\begin{smallmatrix}0\;1\;0\;\cdots 1\;1\end{smallmatrix}]$, and $\sigmav_4=[\begin{smallmatrix}1\;0\;0\;\cdots 1\;1\end{smallmatrix}]$. The task concludes by returning the platform to its initial configuration through $\sigmav_5 =[\begin{smallmatrix}1\;1\;1\;\cdots 0\end{smallmatrix}]$.
 Notably, between $7\si{s}$ and $13\si{s}$, the output $q_3$ (associated with $\sigmav_1$ and $\sigmav_3$) continues its exponential convergence to the desired value seamlessly, despite the switching at $t = 9\si{s}$. Similarly, between $13\si{s}$ and $18\si{s}$, the output $q_1$ (associated with $\sigmav_4$ and $\sigmav_5$) maintains its convergence even during the switching at $t = 18\si{s}$.
In contrast, the outputs $x_{G_2}$ and $y_{G_2}$ (as do $q_2$ and $q_3$ in the second interval) replace the previously active outputs $x_{G_1}$ and $y_{G_1}$ (or $x_{G_2}$ and $y_{G_2}$ in the second interval), which were converging toward the position of the preceding item. Each output that becomes active inherits its initial conditions from its behavior under the previous mode.
The simulation video is available at \mbox{{\small\url{https://youtu.be/JAGs4GiTIg0}}}.

\section{Discussion}


Switching among feedback-linearizing output selections differs from classical switched systems because each meld induces its own nonlinear coordinate chart. A switch therefore produces a state-dependent coordinate mismatch that cannot be neutralized by dwell time alone. The reference-compatibility constant  makes this effect explicit: dwell time attenuates transient tracking errors within a meld, but the achievable invariant tube remains lower-bounded by cross-meld reference mismatch on the common validity set. Dwell time must be chosen together with  the reference trajectories so that consecutive meld references are mutually compatible where their charts overlap. The analysis is inherently local and assumes  exact feedback linearization with full vector relative degree, possibly after dynamic input extension. As a result, operation near singularities and configurations admitting nontrivial zero dynamics are excluded by construction.  If switching occurs near singularities (loss of relative degree) or with insufficient dwell-time, boundedness may be lost (cf.~\cite{TomlinSastry1998,Liberzon2003}). The proposed dwell-time condition and the restriction to the safe set $\mathcal{K}$
prevent such loss within the prescribed operating region. Robustness to uncertainty can be handled via standard robust feedback-linearization arguments, but explicit uncertainty-dependent bounds for 
$L^\Theta, L^\Psi$, and $N$ are left for future work.
\section{Conclusions}\label{subs:conclusion}

This work analyzes switching among feedback-linearizing output selections (melds) and showed that the key difficulty is the change of coordinate charts at switching. By introducing the safe set (nonempty intersection of validity sets) and the reference-compatibility constant, and by quantifying the cross-chart effects via local Lipschitz bounds, explicit dwell-time conditions were derived that guarantee: (i) exponential decay of the active-output tracking errors between switches, (ii) seamless tracking for outputs shared by consecutive melds, and (iii) uniform boundedness and invariance within a common safe set. The resulting conditions are constructive and apply to general nonlinear systems admitting exact feedback linearization.

\textbf{Future directions.}
(1) Partial linearization and zero dynamics.
(2) Singularity avoidance using guards/barrier certificates.
(3) Robust bounds for $L^{\Psi},L^{\Theta}$ and $N$.
(4) Empirical validation  against optimization/null-space prioritization on hardware.

\bibliographystyle{plain}
\bibliography{Bib/bibAlias,Bib/bibAF,Bib/bibCustom,Bib/ref2}

\end{document}